%% file: tare17.tex
\pdfoutput=1
\documentclass[10pt]{article}

\usepackage[preprint]{tmlr}

\usepackage{amsmath,amssymb,amsthm,mathtools}
\usepackage{url}
\usepackage[dvipsnames]{xcolor}
\usepackage{enumitem}
\usepackage{booktabs}
\usepackage{etoolbox}
\AtBeginEnvironment{table}{\setlength{\belowcaptionskip}{6pt}}
\usepackage{graphicx}
\usepackage{thmtools}
\usepackage{hyperref}
\usepackage{cleveref}

\newtheorem{theorem}{Theorem}[section]
\newtheorem{proposition}[theorem]{Proposition}
\newtheorem{corollary}[theorem]{Corollary}
\newtheorem{lemma}[theorem]{Lemma}
\theoremstyle{definition}
\declaretheorem[style=definition, sibling=theorem, name=Definition]{definition}
\theoremstyle{remark}
\newtheorem{remark}[theorem]{Remark}

\newcommand{\R}{\mathbb{R}}

\newcommand{\by}{\mathbf{y}}
\newcommand{\btheta}{\boldsymbol{\theta}}
\newcommand{\bdelta}{\boldsymbol{\delta}}
\newcommand{\bw}{\mathbf{w}}
\newcommand{\diag}{\operatorname{diag}}

\title{Influence Is the Target-Response Jacobian Times Residuals}

\author{\name Paul Geertsema \email paul.geertsema@vlerick.com \\
      \addr Vlerick Business School
      \AND
      \name Simon Geertsema \email simon.geertsema@universiteitleiden.nl \\
      \addr Leiden University}

\def\month{MM}
\def\year{YYYY}
\def\openreview{\url{https://openreview.net/forum?id=XXXX}}

\begin{document}
\maketitle

\begin{abstract}
We introduce the target-response Jacobian
$J = [\partial \hat y_i/\partial y_j]$, which measures how model
predictions respond to changes in individual training targets. $J$
complements training-data attribution (TDA) methods, which characterize
sensitivity to the \emph{weights} placed on training observations in
the loss rather than to their \emph{targets}, and can be computed
matrix-free at the same cost as influence-based TDA methods. We show
that sample-weight influence $\mathcal{I}$ and $J$ follow from the same
implicit differentiation argument and differ only by a per-example multiplier:
for pointwise losses, $\mathcal{I} = J\operatorname{diag}(m)$ exactly
at any stationary point with an invertible Hessian. For canonical-link
losses (squared error, logistic, Poisson) the multiplier is the
training residual, $\mathcal{I} = J\operatorname{diag}(r)$, and softmax
cross-entropy obeys a multiclass version. The identity predicts three
ways in which influence fails when used in place of $J$: small
residuals shrink its predicted shift towards zero (interpolation
collapse), residuals of mixed sign wash out its direction (direction
washout), and under softmax cross-entropy it collapses a training
point's per-class sensitivities towards its own true class (true-class
collapse). We score both operators against refitting with a
per-prediction $R^2$ across target perturbations, with no rescaling. On
tanh networks trained with squared error (Friedman \#1) and logistic
loss (MAGIC, HIGGS), with $N$ from 2,000 to $10^6$, $J$ reaches $R^2$
between 0.99 and 1.00, whereas influence's $R^2$ is negative on every
dataset. On an MNIST softmax classifier, influence's $R^2$ is again
negative, while $J$ reaches 0.92. On HIGGS, one row of $J$, holding all
$10^6$ sensitivities of one prediction, takes under 0.2 seconds on an
RTX 3090 after a one-time setup of under five minutes, whereas
refitting would take $10^6$ refits.
\end{abstract}

\section{Introduction}
\label{sec:intro}

This paper introduces the target-response Jacobian
\begin{equation}
  J \;=\; \bigl[\,\partial \hat y_i / \partial y_j\,\bigr],
\end{equation}
which measures how a model's predictions respond to changes in
individual training targets without retraining. $J$ complements
existing training-data attribution (TDA) methods, which measure how a
prediction depends on the weight given to each training example in
the loss. $J$ measures how it depends on the training targets
themselves, which no existing method does for general differentiable
learners. We show
that $J$ can be computed matrix-free at scale. For $S$ predictions and $N$ training points, $J$
is an $S\times N$ matrix whose row $i$ gives one prediction's
sensitivity to every training target and whose column $j$ gives the
effect of one training target on every prediction. At a fitted model,
implicit differentiation of the stationarity condition splits $J$ into
three standard factors,
\begin{equation}
J
=
\underbrace{G}_{\parbox{2.5cm}{\centering
predictions\\ $\longleftarrow$\\ parameters}}
\underbrace{(-H^{-1})}_{\parbox{2.7cm}{\centering
parameters\\ $\longleftarrow$\\ loss gradient}}
\underbrace{C}_{\parbox{2.5cm}{\centering
loss gradient\\ $\longleftarrow$\\ targets}} ,
\label{eq:jfactorisation_pipeline}
\end{equation}
where $G$ stacks the gradients of the predictions with respect to the
parameters, $H$ is the Hessian of the training objective at the fitted
parameters, and column $j$ of $C$ is the derivative of the
training-loss gradient with respect to $y_j$. Each factor is available
matrix-free in any model trained by gradient methods: $G$ costs a
Jacobian-vector product, $H^{-1}$ an iterative solve against
Hessian-vector products, and each column of $C$ one mixed gradient. $J$
can therefore be computed for most models trained by gradient methods,
and wherever influence functions can be computed. Throughout, a
\emph{target change} $\Delta\by$ is a change in the training targets,
and a \emph{shift} is the resulting change in a prediction, either
after refitting or as predicted by $J$ or influence. A
\emph{perturbation} is a randomly drawn target change.

Sample-weight influence \citep{koh2017understanding} runs the same
pipeline with a different last factor. It uses $B$, the matrix of
per-example loss gradients, in place of $C$, so that
$\mathcal{I}=-GH^{-1}B$. For a pointwise scalar loss, column $j$ of $B$
is column $j$ of $C$ multiplied by a single number, and from this
observation we prove a new
identity: for canonical-link losses,%
\footnote{A canonical-link loss has the form
$\ell(f,y)=h(f)-yf+g(y)$, with $f$ the model output. Squared error,
logistic loss on logits and Poisson loss on log-rates are the common
cases (Corollary~\ref{cor:canonical_link}).}
sample-weight influence equals $J$
times a diagonal matrix of training residuals (the
\emph{canonical-link identity}),
\begin{equation}
  \mathcal{I} \;=\; J\,\diag(r),
  \label{eq:intro_canonical}
\end{equation}
exactly, at any stationary point of the training objective with an
invertible Hessian (Corollary~\ref{cor:canonical_link}).%
\footnote{Read in terms of targets, the identity says that upweighting
training point $j$ by $\epsilon$ has, to first order, the same effect
as moving its target by $\epsilon r_j$, away from the model's fit, and
that deleting the point corresponds to moving its target onto the fit.
Used as a target sensitivity, influence therefore predicts the shift
for the target change $\diag(r)\Delta\by$ rather than for $\Delta\by$.}
It follows that influence cannot stand in for $J$. As a model fits its
training data more closely, the residuals $r_j$ shrink, so
influence's predicted shift goes to zero even where $J$'s does not.

The canonical-link identity \eqref{eq:intro_canonical} is a special
case of a more general \emph{modulation identity},
\begin{equation}
  \mathcal{I} \;=\; J\,\diag(m),
  \qquad
  m_j \;=\; \ell_f / \ell_{fy},
  \label{eq:intro_identity}
\end{equation}
where $\ell_f$ and $\ell_{fy}$ are the first and mixed partial
derivatives of the per-example loss (Theorem~\ref{thm:scalar_mod}). For
canonical-link losses (squared error, logistic, Poisson) the multiplier
$m_j$ reduces to the residual $r_j$, which recovers
$\mathcal{I} = J\,\diag(r)$. Softmax cross-entropy obeys a multiclass
version with one residual per class
(Proposition~\ref{prop:multiclass}). In terms of the
factorisation \eqref{eq:jfactorisation_pipeline}, column $j$ of $B$ is
$m_j$ times column $j$ of $C$. For canonical-link losses, $J$ is
therefore influence with the residual factor removed. Column $j$ of $C$
is then $-1/N$ times the gradient of the model output at $x_j$, so if
an influence implementation uses these scaled output gradients in place
of the per-example loss gradients, it computes $J$. The identity also does
not require the exact Hessian: it holds unchanged for any shared
invertible surrogate, so the damped Gauss--Newton Hessian, which the
experiments also report, inherits it (Remark~\ref{cor:surrogate_mod}).

In the absence of $J$, influence-type methods are the closest available
proxy for applications that need target sensitivity, and they are used
that way: \citet{richardson2023addremove}, for example, rank training
points by their sample-weight influence on a fairness metric and then
flip the labels of the highest-ranked points. The modulation identity
implies three distinct ways in which influence fails in that role. Small
residuals shrink its predicted shift towards zero
(\emph{interpolation collapse}, Corollary~\ref{cor:interp_collapse});
residuals of mixed sign wash out its predicted direction
(\emph{direction washout}, Corollary~\ref{cor:sign_washout}); and under
softmax cross-entropy it collapses the per-class target sensitivities
of a training point onto a single direction, towards that training
point's own true class, whatever target change is asked about
(\emph{true-class collapse}, Corollary~\ref{cor:true_class_collapse}).

Figure~\ref{fig:california_map} shows what this means for a concrete
question: how would a house-price model's predictions change if the
top-coded values in California Housing had been recorded 20\% higher?
$J$ reproduces the refit, which concentrates the shift mainly in two
regions, whereas influence spreads a smaller shift along the coast and
predicts only a small part of the shift in the block groups that move
most
(Section~\ref{sec:exp_california}). None of this is a flaw in influence
as such. It correctly measures what it was built to measure, the
sensitivity of a prediction to per-example loss weights, but that
sensitivity is not target sensitivity, and the modulation identity
explains why.

\begin{figure}[t]
\centering
\includegraphics[width=\linewidth]{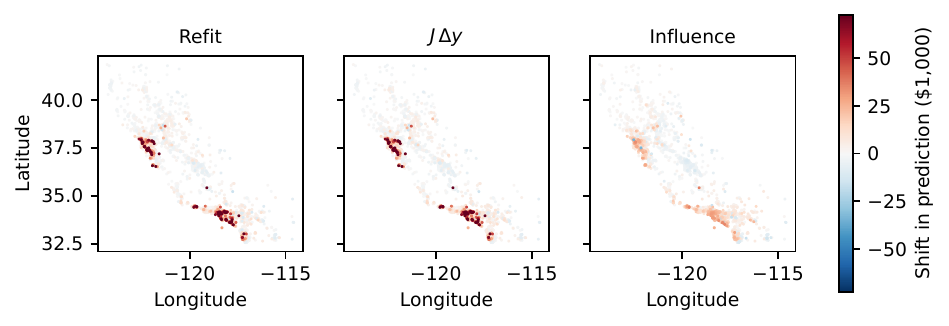}
\caption{What if the top-coded house values in California Housing had
been recorded 20\% higher? Each dot is one of 5,160 held-out block
groups, coloured by the shift in its prediction after the target change
(one colour scale for all panels). The refit (left)
concentrates the shift mainly around the San Francisco Bay Area and
Los Angeles, and $J\Delta\by$ (centre), computed without refitting,
reproduces it. Influence (right) spreads a smaller shift along the
coast and predicts only a small part of the shift in the block groups
that move most.
Section~\ref{sec:exp_california} gives the details.}
\label{fig:california_map}
\end{figure}

Testing $J$, and these predictions about influence, requires a metric
that matches what a target-sensitivity estimator claims, namely that the shift it predicts
in a prediction equals the shift that refitting produces. No standard
metric in the data-attribution literature tests that claim directly.
Pearson correlation tests direction; any affine-rescaled $R^2$ tests
direction after a free rescaling; and rank-correlation scores,
including the linear datamodelling score
\citep{ilyas2022datamodels,park2023trak}, test only ordering. All three
ignore scale and so cannot detect interpolation collapse, which acts on
scale. The
per-prediction coefficient of determination $R^2_i$, which is exactly
\verb|r2_score| from \verb|scikit-learn|, compares the predicted shift
with the refit shift directly, scale included. What
differs from textbook usage is what varies. There, $R^2$ is computed
across the instances of a dataset; here, $R^2_i$ is computed across
random target perturbations applied at a \emph{single} prediction $i$
(Section~\ref{sec:r2}). Its Mincer--Zarnowitz decomposition
\citep{mincer1969evaluation} into direction $\rho_i$, scale $\alpha_i$
and bias $\beta_i$ then makes each of the three failure modes visible.
Interpolation collapse lowers $\alpha_i$; direction washout drives
$\rho_i$ to zero; and true-class collapse lowers both. It lowers
$\rho_i$ because influence points towards the training point's own
true class whichever class its label is moved to, and it shrinks
$\alpha_i$ as the probability the classifier assigns to that training
point's true class approaches one.

The experiments test these predictions against refitting. On tanh networks
trained with squared error (Friedman\,\#1, $N=2{,}000$) and with
logistic loss (MAGIC, $N=14{,}265$, and HIGGS, $N=10^6$), $J$ attains a
per-prediction $R^2$ between 0.99 and 1.00, with direction and scale
near one, whereas influence's $R^2$ is negative on every dataset
(Section~\ref{sec:exp_tabular}, Table~\ref{tab:r2_mlp}). Each predicted
failure appears where the identity predicts it. Influence's
direction is washed out on all three datasets, with $\rho$ between
$-0.08$ and $0.03$, and on the two logistic datasets its scale is
shrunk ($\alpha$ of 0.30 on MAGIC and 0.47 on HIGGS).

On a softmax classifier for MNIST, with $N=2{,}000$ as on
Friedman,
influence's $R^2$ is negative ($-0.20$), with a weak direction
($\rho$ of $-0.37$) and a shrunk scale ($\alpha$ of 0.10), as
true-class collapse predicts, while $J$ reaches $R^2$ of 0.92
(Section~\ref{sec:exp_mnist},
Table~\ref{tab:r2_mlp}).

As an application, Section~\ref{sec:exp_california} asks how the
model's predictions would change if the top-coded house values in
California Housing had been recorded higher
(Figure~\ref{fig:california_map}), and checks both answers against a
refit. For increases from 5\% to 50\%, $J$'s $R^2$ lies between 0.98 and 1.00,
while influence's is at most 0.27 (Section~\ref{sec:exp_california},
Table~\ref{tab:california}).

The implementation never forms $J$. On HIGGS, with $N=10^6$ training points,
one row of $J$, which gives one prediction's sensitivities to every
training target, takes under 0.2 seconds on a single RTX~3090 after a
one-time setup of under five minutes, whereas
producing the same row by refitting would take $10^6$ refits, one per
training target (Appendix~\ref{app:timing}).

\section{Related work}
\label{sec:related}

\begin{table}[t]
\centering
\small
\caption{Training-point attribution operators, grouped by what is
differentiated. Panel~A collects operators of sensitivity to targets,
through which predictions are linear in the targets: exactly for the
classical smoothers and AXIL, and to first order for general learners.
Panel~B covers sensitivity to sample weights. The modulation identity
links the two panels.}
\label{tab:lineage}
\setlength{\tabcolsep}{4pt}
\begin{tabular}{@{}llll@{}}
\toprule
\multicolumn{4}{c}{\textbf{Panel A.  Sensitivity to targets}} \\
\midrule
Object & Identity & Source & Literature \\
\midrule
Linear regression                  & $\hat\by = K_{\mathrm{OLS}}\,\by$ exactly     & \citet{hoaglin1978hat}     & Statistics \\
Kernel ridge regression            & $\hat\by = K_{\mathrm{KRR}}\,\by$ exactly     & \citet{saunders1998ridge}  & Smoothers \\
L2 gradient boosting, fixed trees  & $\hat\by = K_{\mathrm{AXIL}}\,\by$ exactly    & \citet{geertsema2026axil}  & XAI \\
\textbf{Target-response Jacobian}  & $\hat\by \approx J\,\by + a^\star$ locally  & Eq.~\eqref{eq:local-identity} & TDA \\
\midrule
\multicolumn{4}{c}{\textbf{Panel B.  Sensitivity to sample weights}} \\
\midrule
Sample-weight influence $\mathcal{I}$  & $\partial \hat y_i/\partial w_j$ at $\bw=\mathbf{1}$ & \citet{koh2017understanding} & TDA \\
Canonical TRAK                         & projected, ensembled $\partial \hat y_i/\partial w_j$ & \citet{park2023trak} & TDA \\
\textbf{Modulation identity}           & $\mathcal{I} = J\,\diag(m)$                 & Thm.~\ref{thm:scalar_mod} & TDA \\
\bottomrule
\end{tabular}
\end{table}

Panel~A of Table~\ref{tab:lineage} reorganises three earlier
literatures, in statistics, smoothers and XAI, around a single
property: that predictions are linear in the training targets.
The OLS hat matrix \citep{hoaglin1978hat} and the kernel-ridge smoother
\citep{saunders1998ridge} have this property exactly, since
$\hat\by = K_{\mathrm{OLS}}\,\by$ and $\hat\by = K_{\mathrm{KRR}}\,\by$.
Regression diagnostics combine two quantities, leverage (the
diagonal of the hat matrix) and residuals. For OLS, where
$J=K_{\mathrm{OLS}}$, \citet{cook1982residuals} write the effect on
prediction $i$ of deleting training point $j$ as
$(K_{\mathrm{OLS}})_{ij}\,r_j/(1-(K_{\mathrm{OLS}})_{jj})$: a column of
$J$ scaled by the residual, with a leverage correction for removing the
point entirely. Deletion is a change in sample weight, so this is the
canonical-link identity for one learner, in finite form.
\citet{pregibon1981logistic} extended these diagnostics to
logistic regression, and \citet{stlaurent1992leverage,stlaurent1993leverage}
to nonlinear models. The closest existing object is the generalized
leverage of \citet{wei1998generalized}, defined as
$\partial\hat\by/\partial\by^\top$ for parametric models fitted by
maximum likelihood. It is $J$ with the evaluation points fixed at the
training inputs. In explainable AI (XAI), AXIL
\citep{geertsema2026axil} extends the smoother identity to fixed-tree
squared-error gradient boosting as $\hat\by = K_{\mathrm{AXIL}}\,\by$.
It also states the increment form
$F(\by^\star+\bdelta) = F(\by^\star) + J\bdelta + o(\|\bdelta\|)$ for
general differentiable learners, but leaves open how accurate that form
is and how to compute it at scale. The $J$ of this paper is the level
form $\hat\by \approx J\,\by + a^\star$ \eqref{eq:local-identity},
which places $J$ alongside the smoothers as the local first-order
operator.

Panel~B of Table~\ref{tab:lineage} covers the other axis, sensitivity
to sample weights. Sample-weight influence, which goes back to the
influence curve of robust statistics \citep{hampel1974influence} and
was brought to modern machine learning by \citet{koh2017understanding},
and canonical TRAK \citep{park2023trak} both compute
$\partial \hat y_i/\partial w_j$.%
\footnote{Other widely used methods also measure sensitivity to sample
weights, and those built on per-example loss gradients contain the
residual factor: for canonical-link losses the representer weight of
training point $j$ \citep{yeh2018representer} is proportional to $r_j$,
and each checkpoint term of TracIn \citep{pruthi2020tracin} contains the
residual of point $j$ at that checkpoint. Datamodels
\citep{ilyas2022datamodels} regress the output on subset inclusion, a
0/1 sample weight.}
The work that comes closest to target sensitivity is on labels.
\citet{koh2017understanding} remark that analogous equations apply to
label changes. \citet{kong2022resolving} derive how relabelling a
training point changes a test loss, and use it to choose training
points to relabel. \citet{adebayo2023quantifying} derive the effect of
label changes on a group fairness metric, and use it to find the label
errors that most affect it. Other papers study particular label changes.
\citet{gong2022partial} compute the parameter change for one candidate
label replacement in partial-label learning,
\citet{zhang2022understanding} measure the test-loss effect of
soft-label changes in programmatic weak supervision, and
\citet{warnecke2023machine} give a closed-form parameter update that
unlearns a specified set of label changes. Each of these papers
computes the effect of one kind of change on one downstream quantity.
$J$, by contrast, gives the first-order response of every prediction to
every training target, and we check it against refits. Through the
modulation identity, $J$ also links the two panels of
Table~\ref{tab:lineage}, and so connects literatures that have
developed separately. The Jacobian itself is standard mathematics. To
our knowledge, however, no prior work states $J$ as the target-response
operator for general learners at evaluation points, computes it at
scale, checks it against refits, or relates it exactly to influence.

\section{The target-response Jacobian and the modulation identity}
\label{sec:identity}

This section defines $J$ and sample-weight influence for a learner
trained to a stationary point, and proves that for pointwise scalar
losses the two operators differ only by a per-example factor, with a
per-class analogue for softmax cross-entropy. Table~\ref{tab:notation}
in Appendix~\ref{app:notation} collects the notation.

\paragraph{Setup.}
Let $X=(x_1,\dots,x_N)$ be training inputs with targets
$\by=(y_1,\dots,y_N)\in\R^N$ and let $Q=(z_1,\dots,z_S)$ be an
evaluation set. Throughout, $i\in[S]$ indexes evaluation predictions
and $j\in[N]$ indexes training points, where $[N]:=\{1,\dots,N\}$.
$\|\cdot\|$ is the Euclidean norm. The learner has parameters
$\btheta\in\R^P$ and minimises
\begin{equation}
  \mathcal{L}(\btheta;\by)
  = \frac{1}{N}\sum_{j=1}^{N}\ell_j(\btheta,y_j) + R(\btheta),
\end{equation}
with each $\ell_j$ twice continuously differentiable in
$(\btheta,y_j)$ and $R$ twice continuously differentiable in $\btheta$
near $\btheta^\star$. Write $\by^\star$ for the observed targets and
$\btheta^\star$ for the fitted parameters. Assume $\btheta^\star$ is an
interior stationary point of $\mathcal{L}(\,\cdot\,;\by^\star)$ at
which the Hessian
$H:=\nabla^2_{\btheta\btheta}\mathcal{L}(\btheta^\star;\by^\star)$ is
invertible. The implicit function theorem then gives a unique
differentiable map $\btheta^\star(\by)$ near $\by^\star$ with
$\nabla_{\btheta}\mathcal{L}(\btheta^\star(\by);\by)=0$ and
$\btheta^\star(\by^\star)=\btheta^\star$. Below, ``trained model'' means
this locally defined solution $\btheta^\star(\by)$, which need not be
a global minimiser.

\begin{definition}[Target-response map]
\label{def:target_response}
The \emph{target-response map} $F:\R^N\to\R^S$ is
\begin{equation}
  F(\by) := \bigl(q_1(\btheta^\star(\by)),\dots,q_S(\btheta^\star(\by))\bigr)^\top,
\end{equation}
where $q_i(\btheta)$ is a differentiable scalar prediction at
evaluation point $z_i$, such as a logit or a predicted probability.
\end{definition}

\begin{definition}[Target-response Jacobian]
\label{def:target_jacobian}
The \emph{target-response Jacobian} is
$J := \dfrac{\partial F}{\partial\by}(\by^\star)\in\R^{S\times N}$.
\end{definition}

Entry $J_{ij}$ is the first-order change in prediction $i$ when
training target $j$ moves by one unit,
\begin{equation}
  F(\by^\star+\bdelta) \;=\; F(\by^\star) + J\bdelta + o(\|\bdelta\|)
  \qquad\text{as }\bdelta\to 0.
  \label{eq:taylor_F}
\end{equation}
Writing $\hat\by:=F(\by)$ for the predictions and
$a^\star:=F(\by^\star)-J\by^\star$, the expansion reads
\begin{equation}
  \hat\by \;\approx\; J\,\by + a^\star,
  \label{eq:local-identity}
\end{equation}
the local version of the smoother identity $\hat\by=K\by$ in
Table~\ref{tab:lineage}. When $F(\by)=K\by$ exactly, $J=K$ and
$a^\star=0$. Vector targets $y_j\in\R^{K_c}$, such as one-hot or soft
labels over $K_c$ classes, make $J$ a 3-tensor in
$\R^{S\times N\times K_c}$ (Proposition~\ref{prop:multiclass}).

The attribution literature uses a different per-example sensitivity,
the sample-weight influence of \citet{koh2017understanding}
(Appendix~\ref{app:notation} relates its sign and scale to other
conventions). To define it, give each example a weight,
$\mathcal{L}(\btheta;\by,\bw)=(1/N)\sum_j w_j\ell_j + R(\btheta)$, let
$F(\by,\bw):=(q_i(\btheta^\star(\by,\bw)))_{i=1}^{S}$, so that
$F(\by)=F(\by,\mathbf{1})$, and set
\begin{equation}
  \mathcal{I} \;:=\; \frac{\partial F}{\partial\bw}(\by^\star,\mathbf{1})
  \;\in\;\R^{S\times N}.
\end{equation}

The identity below holds for pointwise scalar losses
$\ell_j(\btheta,y_j)=\ell\bigl(f_{\btheta}(x_j),y_j\bigr)$ with
$f_{\btheta}(x_j)\in\R$ and $\ell$ twice continuously differentiable in
$(f,y)$. Write $f_j:=f_{\btheta^\star}(x_j)$ and let $\ell_f$ and
$\ell_{fy}$ be the partials of $\ell$ at $(f_j,y_j^\star)$. The
\emph{modulation factor} is
\begin{equation}
  m_j \;:=\; \ell_f(f_j,y_j^\star)\,/\,\ell_{fy}(f_j,y_j^\star),
  \qquad j\in[N],
  \label{eq:mu_def}
\end{equation}
which requires $\ell_{fy}(f_j,y_j^\star)\neq 0$ at every training point.

\begin{theorem}[Modulation identity]
\label{thm:scalar_mod}
Under the conditions above, with $\ell_{fy}(f_j,y_j^\star)\neq 0$ for
every $j\in[N]$, exactly at the trained model $\btheta^\star$,
\begin{equation}
  \mathcal{I} \;=\; J\,\diag(m).
  \label{eq:modulation}
\end{equation}
\end{theorem}

\begin{proof}
\emph{Step 1 (implicit differentiation).} Let
$g_i := \nabla_{\btheta} q_i(\btheta^\star)$ and
$G := [g_1\cdots g_S]^\top \in \R^{S\times P}$. Differentiating the
stationarity condition $\nabla_{\btheta}\mathcal{L}(\btheta^\star(\by);\by)=0$
in $y_j$ gives
\[
  H\,\frac{\partial\btheta^\star}{\partial y_j} + c_j \;=\; 0,
  \qquad
  c_j := \nabla_{\btheta}(\partial\mathcal{L}/\partial y_j)\big|_{(\btheta^\star,\by^\star)}
       = (1/N)\nabla_{\btheta}(\partial\ell_j/\partial y_j),
\]
so $\partial\btheta^\star/\partial y_j = -H^{-1}c_j$ and
$\partial q_i/\partial y_j = -g_i^\top H^{-1}c_j$. With
$C := [c_1\cdots c_N] \in \R^{P\times N}$,
\begin{equation}
  J \;=\; -\,G\,H^{-1}\,C.
  \label{eq:jfactorisation}
\end{equation}
Differentiating instead in $w_j$ at $\bw=\mathbf{1}$ gives
$H\,\partial\btheta^\star/\partial w_j + b_j = 0$ with
$b_j := (1/N)\nabla_{\btheta}\ell_j(\btheta^\star,y_j^\star)$, so with
$B := [b_1\cdots b_N] \in \R^{P\times N}$,
\begin{equation}
  \mathcal{I} \;=\; -\,G\,H^{-1}\,B.
  \label{eq:influence_factorisation}
\end{equation}

\emph{Step 2 (per-example chain rule).} Write
$\psi_j := \nabla_{\btheta}f_{\btheta^\star}(x_j)$. The chain rule on
$\ell(f_{\btheta}(x_j),y_j)$ gives
$\nabla_{\btheta}\ell_j = \ell_f\,\psi_j$ and
$\nabla_{\btheta}(\partial\ell_j/\partial y_j) = \ell_{fy}\,\psi_j$, so
$b_j = (1/N)\ell_f\,\psi_j$ and $c_j=(1/N)\ell_{fy}\,\psi_j$. Both are
multiples of $\psi_j$, and because $\ell_{fy}\neq 0$, $b_j=m_j c_j$
with $m_j=\ell_f/\ell_{fy}$. Stacking the columns over $j$ gives
\begin{equation}
  B \;=\; C\diag(m).
  \label{eq:B_Cmu}
\end{equation}

\emph{Step 3 (combine).} Substituting \eqref{eq:B_Cmu} into
\eqref{eq:influence_factorisation} and using \eqref{eq:jfactorisation},
\begin{equation}
  \mathcal{I} \;=\; -GH^{-1}C\,\diag(m) \;=\; J\,\diag(m),
  \label{eq:modulation_proof}
\end{equation}
as required.
\end{proof}

Equation~\eqref{eq:jfactorisation} is the pipeline
\eqref{eq:jfactorisation_pipeline} of the introduction, and influence
runs the same pipeline with $B$ in place of $C$.

\begin{corollary}[Canonical-link losses]
\label{cor:canonical_link}
Suppose $\ell$ has the canonical exponential-family form
$\ell(f,y) = h(f) - y\,f + g(y)$, with $h$ and $g$ twice continuously
differentiable. Then $\ell_f=h'(f)-y$ and $\ell_{fy}=-1$, so the
hypothesis of Theorem~\ref{thm:scalar_mod} holds for any fitted model,
\begin{equation}
  m_j \;=\; y_j^\star - h'(f_j) \;=:\; r_j
\end{equation}
is the training residual, and $\mathcal{I}=J\,\diag(r)$.
\end{corollary}

The residual is $y_j-f_j$ for squared error, $y_j-\sigma(f_j)$ for
logistic loss and $y_j-e^{f_j}$ for Poisson regression. Influence is
therefore zero for a point the model fits exactly, whatever that
point's effect on the predictions when its target moves.

\begin{remark}[Surrogate Hessians]
\label{cor:surrogate_mod}
Step~2 of the proof does not involve $H$. The identity therefore also
holds for $\hat{\mathcal I}:=-G\hat H^{-1}B$ and $\hat J:=-G\hat H^{-1}C$
with any invertible $\hat H$ used for both, $\hat{\mathcal I}=\hat J\,\diag(m)$.
\end{remark}

Theorem~\ref{thm:scalar_mod} needs a scalar model output and a scalar
target. For softmax
cross-entropy the output is a vector of $K_c$ logits,
$f_{\btheta}:\R^d\to\R^{K_c}$, with probabilities
$p_k(\btheta,x):=e^{f_k}/\sum_{k'}e^{f_{k'}}$ and per-example loss
$\ell_j(\btheta,y_j)=\log\sum_{k'}e^{f_{k'}(\btheta,x_j)}-y_j^\top f(\btheta,x_j)$
on soft or one-hot targets with $\sum_k y_{jk}=1$, and the identity
takes a tensor form, with one residual per class.

\begin{proposition}[Softmax cross-entropy]
\label{prop:multiclass}
Let $r_{jk}:=y_{jk}-p_k(\btheta^\star,x_j)$ and let
$J\in\R^{S\times N\times K_c}$ with
$J_{i,j,k}:=\partial F_i(\by^\star)/\partial y_{jk}$. Then
\begin{equation}
  \mathcal{I}_{:,j} \;=\; \sum_{k=1}^{K_c} r_{jk}\, J_{:,j,k}
  \;\in\;\R^S.
  \label{eq:multiclass}
\end{equation}
The same holds when both operators share any invertible $\hat H$.
\end{proposition}

\begin{proof}
Appendix~\ref{app:multiclass}.
\end{proof}

\section{Implications for influence}
\label{sec:failures}

Suppose influence is used to predict how predictions change when
training targets change. The modulation identity then predicts three
ways in which influence fails, even though it still measures
sensitivity to sample weights correctly.

The first failure concerns scale. When the modulation factors shrink,
as the residuals do when a model fits its training data more closely,
influence shrinks with them even while $J$ stays bounded.

\begin{corollary}[Interpolation collapse]
\label{cor:interp_collapse}
Under Theorem~\ref{thm:scalar_mod}, with $N$ and $S$ fixed, suppose
that along a sequence of fitted models $\max_j|m_j|\to 0$ while the
entries of $J$ stay bounded. Then every entry of $\mathcal{I}$ goes to
zero, and so does influence's predicted shift $\mathcal{I}\Delta\by$
for every bounded perturbation $\Delta\by$. As a target-sensitivity
estimator, influence collapses to predicting no change. The same holds
for any shared $\hat H$, with the bound imposed on the entries of
$\hat J$.
\end{corollary}

\begin{proof}
$|\mathcal{I}_{ij}|=|m_j|\,|J_{ij}|\le\max_l|m_l|\,|J_{ij}|\to 0$, and
$(\mathcal{I}\Delta\by)_i=\sum_j\mathcal{I}_{ij}\Delta y_j$ is a sum of
$N$ such entries times bounded factors.
\end{proof}

Kernel ridge regression shows the collapse in closed form. As the ridge
parameter goes to zero the model interpolates, the residuals vanish and
influence goes to zero, while $J$ converges to the interpolating
smoother (Appendix~\ref{app:krr}).

The second failure concerns direction. Draw a perturbation by choosing
a training point $U$ uniformly from $[N]$ and moving its target by a
random amount $\delta$, $\Delta\by=\delta\,e_U$. At prediction $i$,
influence then predicts the shift $r_U J_{iU}\delta$ and $J$ predicts
$J_{iU}\delta$.

\begin{corollary}[Direction washout]
\label{cor:sign_washout}
Under Corollary~\ref{cor:canonical_link}, fix a prediction $i$ with
$J_{i,:}\neq 0$, and let $\delta$ be independent of $U$ with
$\mathbb{E}[\delta]=0$ and $\mathbb{E}[\delta^2]>0$. Weight training
point $j$ by $\pi_{ij}:=J_{ij}^2/\sum_{l}J_{il}^2$, and let $\mu_i$ and
$\sigma_i^2$ be the mean and variance of the residuals under these
weights, with $\mu_i^2+\sigma_i^2>0$. Then the correlation between the
shifts predicted by influence and by $J$ is
\begin{equation}
  \operatorname{corr}\bigl(r_U J_{iU}\delta,\; J_{iU}\delta\bigr)
  \;=\; \frac{\mu_i}{\sqrt{\mu_i^2+\sigma_i^2}}.
  \label{eq:washout}
\end{equation}
It is zero when the weighted mean residual is zero, and close to zero
whenever that mean is small relative to the spread of the residuals.
\end{corollary}

\begin{proof}
Both shifts have mean zero because $\mathbb{E}[\delta]=0$ and
$\delta$ is independent of $U$. Their covariance is
$\mathbb{E}[r_UJ_{iU}^2]\,\mathbb{E}[\delta^2]$, and their variances are
$\mathbb{E}[r_U^2J_{iU}^2]\,\mathbb{E}[\delta^2]$ and
$\mathbb{E}[J_{iU}^2]\,\mathbb{E}[\delta^2]$. Dividing numerator and
denominator by $\mathbb{E}[J_{iU}^2]\,\mathbb{E}[\delta^2]>0$ turns the
expectations over $U$ into weighted means,
$\mathbb{E}_\pi[r]=\mu_i$ and $\mathbb{E}_\pi[r^2]=\mu_i^2+\sigma_i^2$.
\end{proof}

The third failure is specific to softmax cross-entropy.

\begin{corollary}[True-class collapse]
\label{cor:true_class_collapse}
Under softmax cross-entropy with one-hot label $y_j = e_{t_j}$, so that
$t_j$ is the recorded class of training point $j$,
Proposition~\ref{prop:multiclass} gives
\begin{equation}
  \mathcal{I}_{:,j}
  \;=\;
  (1-p_{j,t_j})\,J_{:,j,t_j}
  \;-\; \sum_{k\neq t_j} p_{j,k}\, J_{:,j,k}.
  \label{eq:true_class_decomp}
\end{equation}
For the label change $\Delta\by_j = \delta(e_a - e_{t_j})$ that moves
mass $\delta$ from the true class $t_j$ to a class $a\neq t_j$, and
with influence scored as $\delta\,\mathcal{I}_{i,j}$
(Appendix~\ref{app:multiclass}), the predicted shifts in prediction $i$
are
\begin{equation*}
  (J\Delta\by)_i \;=\; \delta\bigl(J_{i,j,a} - J_{i,j,t_j}\bigr),
  \qquad
  (\mathcal{I}\Delta\by)_i \;=\; \delta\,\mathcal{I}_{i,j}.
\end{equation*}
$J$'s shift depends on the destination class $a$, whereas influence's
does not.
\end{corollary}

\begin{proof}
Substitute the one-hot residuals $r_{j,t_j}=1-p_{j,t_j}$ and
$r_{j,k}=-p_{j,k}$ for $k\neq t_j$ into \eqref{eq:multiclass}.
$J$'s shift is $\sum_k J_{i,j,k}\Delta y_{jk}$. $\mathcal{I}_{i,j}$
depends on $i$, $j$ and the fitted model but not on $a$.
\end{proof}

Softmax probabilities are positive, so in \eqref{eq:true_class_decomp}
the true class $t_j$ of training point $j$ receives the only positive
weight, $1-p_{j,t_j}$, and every other class $k$ a negative weight
$-p_{j,k}$; these sum to $-(1-p_{j,t_j})$. Influence's column for a
training point therefore always points towards that point's own true
class, whatever target change is asked about. It gives the same
prediction whichever class the label moves to, and so contains no
information about where the label goes. Its size shrinks as well: the
whole column is $1-p_{j,t_j}$ times a weighted difference of columns of
$J$ (Appendix~\ref{app:multiclass}), so as the probability the
classifier assigns to $t_j$ approaches one, influence's shift goes to
zero.

The same reasoning suggests a problem for canonical TRAK
\citep{park2023trak}, whose formula weights each training example by
$\partial\ell/\partial f_j$. For canonical-link losses this weight is
$-r_j$. With its projection and ensembling held fixed, TRAK therefore
contains the same residual factor, and we expect it to share the three
failures. We do not test TRAK empirically: its reference implementation
supports only multi-class tasks (image, text and CLIP classification),
and applying it to the squared-error and single-logit models of our
tabular experiments would require model-output code that does not yet
exist.%
\footnote{For multi-class classification, as on MNIST, TRAK's weight
for training point $j$ is $1-p_{j,t_j}$, the prefactor of
Corollary~\ref{cor:true_class_collapse}, and it gives one score per
pair of prediction and training point. Its predicted effect of a label
change is therefore the same whichever class the label moves to, as for
influence.}

\section{Evaluating a target-sensitivity estimator}
\label{sec:r2}

The experiments compare $J$ and influence with refits, and so need a
score for how well a matrix predicts the effect of a target change. A
target-sensitivity estimator is a matrix $M\in\R^{S\times N}$ used to
predict the refit shift, $\Delta F_i \approx (M\Delta\by)_i$. Both $J$
and $\mathcal{I}$ are scored this way, with no rescaling.

\begin{definition}[Per-prediction $R^2$]
\label{def:r2}
For prediction $i$ and a perturbation distribution $\mathcal{D}$ on
$\R^N$ ($\R^{NK_c}$ for vector targets),
\begin{equation}
  R^2_{i}(M;\mathcal{D})
  \;:=\; 1 \;-\;
  \frac{\mathbb{E}_{\Delta\by\sim\mathcal{D}}
         \!\bigl[\bigl(\Delta F_i(\Delta\by) - (M\Delta\by)_i\bigr)^2\bigr]}
       {\operatorname{Var}_{\Delta\by\sim\mathcal{D}}
         \!\bigl[\Delta F_i(\Delta\by)\bigr]},
  \label{eq:r2}
\end{equation}
where $\Delta F_i(\Delta\by) := F_i(\by^\star+\Delta\by) - F_i(\by^\star)$
is the refit shift. The reported score is the median of $R^2_i$ over the
evaluation set.
\end{definition}

Unlike the usual regression $R^2$, which is taken across data points,
this one is taken across perturbations at a single prediction. In
practice we draw
$K$ perturbations
from $\mathcal{D}$ and refit once for each, so every
prediction has $K$ pairs of refit and predicted shifts, and $R^2_i$ is
\texttt{sklearn.metrics.r2\_score} \citep{pedregosa2011scikit} on those
pairs. Predicting the mean refit shift for every perturbation scores
zero, and an estimator that gets the direction wrong scores below zero.
We report the median over predictions rather than the mean, which a
few failing predictions would otherwise dominate.

At prediction $i$, let $\rho_i$ be the correlation between predicted
and refit shifts over $\mathcal{D}$, $\alpha_i$ the standard deviation
of the predicted shift divided by that of the refit shift, and
$\beta_i$ the difference of their means divided by the standard
deviation of the refit shift. The Mincer--Zarnowitz identity
\citep{mincer1969evaluation} splits the score into direction, scale and
bias (Lemma~\ref{lem:mincer_zarnowitz}),%
\footnote{In the Mincer--Zarnowitz regression of refit shifts on
predicted shifts, $\alpha$ and $\beta$ usually denote the intercept and
the slope. Here $\alpha_i$ is a scale and $\beta_i$ a standardised bias,
as defined above.}
\begin{equation}
  R^2_{i}
  \;=\;
  2\rho_i\alpha_i \;-\; \alpha_i^2 \;-\; \beta_i^2.
  \label{eq:r2_decomp}
\end{equation}
Among the failures of Section~\ref{sec:failures}, interpolation
collapse drives $\alpha_i$ to zero, direction washout drives $\rho_i$
to zero, and true-class collapse lowers both. Pearson
correlation and rank-correlation scores, including the linear
datamodelling score \citep{ilyas2022datamodels,park2023trak}, ignore
scale and so cannot detect interpolation collapse. We therefore report
$R^2_i$ with its three components beside it.

\section{Experiments}
\label{sec:experiments}

Every experiment follows the same design. It trains a base model to
stationarity, draws $K$ perturbations of training targets, refits each
one by warm-started L-BFGS from the base parameters, and scores $J$ and
influence against the same refits. Warm-starting matters because it
targets the locally defined solution $\btheta^\star(\by)$ that $J$
differentiates, whereas retraining from a fresh initialisation can
reach a different minimum. Every fit, base or refit, stops on the same
gradient-norm rule. Both operators use the exact Hessian of the
training objective, so neither score includes the error of a surrogate
Hessian. Tables~\ref{tab:tabular_diag} and~\ref{tab:mnist_diag} also
report $J$ with the damped Gauss--Newton Hessian of
Remark~\ref{cor:surrogate_mod}.
Appendix~\ref{app:protocol} gives the architectures, perturbation
distributions and the stopping rule. Kernel ridge regression, where
$J$ is the smoother matrix and scores $R^2=1.000$ exactly, is in
Appendix~\ref{app:krr} (Table~\ref{tab:r2_krr}).

\subsection{Tabular data}
\label{sec:exp_tabular}

Table~\ref{tab:r2_mlp} reports one-hidden-layer tanh networks on
Friedman\,\#1 \citep{friedman1991multivariate} with squared error and
$N=2{,}000$, MAGIC Gamma \citep{bock2004magic} with logistic loss and
$N=14{,}265$, and HIGGS \citep{baldi2014searching} with logistic loss
and $N=10^6$. Entries are medians over evaluation predictions,
averaged over five seeds. The random row gives a floor.

\begin{table}[t]
\centering
\caption{Per-prediction $R^2$ \eqref{eq:r2} and its direction, scale and
bias components \eqref{eq:r2_decomp} against refits. Medians over
predictions, mean $\pm$ SD over five seeds. $K=100$ perturbations
per seed, and $K=20$ on MNIST. Both operators use the exact Hessian.
Random applies one shared permutation
to the perturbation axis of the refit shifts, which breaks the pairing
while keeping each prediction's distribution of shifts. Each column is
summarised separately, so the entries need not satisfy
\eqref{eq:r2_decomp} exactly. Bold marks the better $R^2$ of $J$ and
influence.}
\label{tab:r2_mlp}
\small
\setlength{\tabcolsep}{4pt}
\input{table_r2_mlp}
\end{table}

On the three tabular datasets of Table~\ref{tab:r2_mlp}, $J$ reaches $R^2$ between 0.99 and 1.00, with direction $\rho$
of at least 0.99, scale $\alpha$ between 1.00 and 1.02 and bias near zero.
Influence's $R^2$, by contrast, is negative on all three datasets,
because its direction is washed out ($\rho$ between $-0.08$ and
$0.03$; Appendix~\ref{app:protocol} covers the class-dependent
perturbations of MAGIC and HIGGS). A model that fits well over-predicts about as much as it
under-predicts, so its residuals average close to zero, and unless
their sign is related to the weights $\pi_{ij}$ of
Corollary~\ref{cor:sign_washout}, the weighted mean $\mu_i$ is small
relative to their spread. With $\rho$ near zero, \eqref{eq:r2_decomp} leaves
$-\alpha^2-\beta^2$, which is negative at any scale. On MAGIC and HIGGS
the scale is also shrunk ($\alpha$ of 0.30 and 0.47), because logistic
residuals lie between $-1$ and $1$. On Friedman the loss is
squared error, whose residual $r=y-f$ has the units of the target. Measuring the
target in units $c$ times smaller therefore multiplies influence's
$\alpha$ by $c$, which means its value there (0.68) is not a measure of
shrinkage. With the Gauss--Newton Hessian, $J$'s $R^2$ is between 0.93
and 0.96 and influence's is still negative
(Table~\ref{tab:tabular_diag}).

\subsection{Softmax classifier on MNIST}
\label{sec:exp_mnist}

MNIST \citep{lecun1998gradient} tests $J$ and true-class collapse on a
softmax classifier. The design follows the
tabular runs. Each seed draws $N=2{,}000$ training images, the same $N$
as on Friedman, and 2,000 test images. It fits the same one-hidden-layer
tanh network of width 64 to the flattened images, so $P$ is about
41,000 (Table~\ref{tab:mnist_diag}). The prediction $q_i$ is the logit
of the base model's predicted class at test image $i$. Each seed draws
$K=20$ label changes $\Delta\by_j=\delta(e_a-e_t)$, with $j$ uniform on
$[N]$, $t$ the training image's true class, $a$ a random other class
and $\delta\sim\mathrm{Unif}[0.01,0.30]$, the range used on MAGIC and
HIGGS. $K$ is smaller than on the tabular data because each refit
takes up to half a minute on average (Appendix~\ref{app:protocol}). With the Gauss--Newton Hessian, $J$'s
$R^2$ is 0.74 (Table~\ref{tab:mnist_diag}).

In Table~\ref{tab:r2_mlp}, $J$ reaches $R^2$ of 0.92 on MNIST, with
direction $\rho$ of 0.97 and scale $\alpha$ of 0.99.
Influence's direction is weak ($\rho$ of $-0.37$), because its
predicted shift is the same whichever class the label moves to, as
true-class collapse predicts (Appendix~\ref{app:multiclass}).%
\footnote{Influence scores a label change as a change in sample weight
(Appendix~\ref{app:multiclass}), and the sign of that correspondence is
a convention. Scoring influence with the opposite sign flips its $\rho$,
leaves its $\alpha$ unchanged, and its $R^2$ stays negative on every
seed.}
Its scale is also shrunk ($\alpha$ of 0.10). The classifier assigns
a training point's true class a probability close to one (median
between 0.92 and 0.93 across seeds; Table~\ref{tab:mnist_diag}), so the
prefactor $1-p_{j,t_j}$ of Corollary~\ref{cor:true_class_collapse} is
small. Influence's predicted shifts are therefore close to zero, and
its $R^2$ is negative ($-0.20$).

\section{Application: Censored house prices}
\label{sec:exp_california}

California Housing \citep{pace1997sparse} gives the median house value
of 20,640 census block groups. Values above \$500{,}000 are top-coded
at \$500{,}001, so the 965 most expensive block groups share one
recorded target. How a fitted model's predictions would change if the
top-coded values were higher is a what-if about training targets.
Increasing the top-coded targets by a fraction
$\kappa$ is the target change $\Delta y_j=\kappa y_j$ on the top-coded
training points, with $\Delta y_j=0$ elsewhere. $J$ predicts the shift
in the predictions as $J\Delta\by$. Influence predicts
$\mathcal{I}\Delta\by=J\diag(r)\Delta\by$, which weights the target
change of each top-coded point by that point's residual. Each of the
two predicted shifts costs one Hessian solve, for all
held-out predictions at once.

We fit the squared-error network of Section~\ref{sec:exp_tabular} to a
random three quarters of the rows ($N=15{,}480$), increase the
top-coded targets by 5\%, 10\%, 20\% and 50\%, refit, and compare the
shift in the 5,160 held-out predictions with the shifts that $J$ and
influence predict. Each increase is a single target change rather than
$K$ random perturbations, so $R^2$ and its components are computed
across the held-out predictions of that target change.
Appendix~\ref{app:california} gives the protocol.

\begin{table}[t]
\centering
\caption{Increasing the top-coded house values of California Housing.
$R^2$ and its components \eqref{eq:r2_decomp} are computed across the
5,160 held-out predictions of each target change. Seed 0.}
\label{tab:california}
\small
\setlength{\tabcolsep}{4pt}
\input{table_california}
\end{table}

\begin{figure}[t]
\centering
\includegraphics[width=0.85\linewidth]{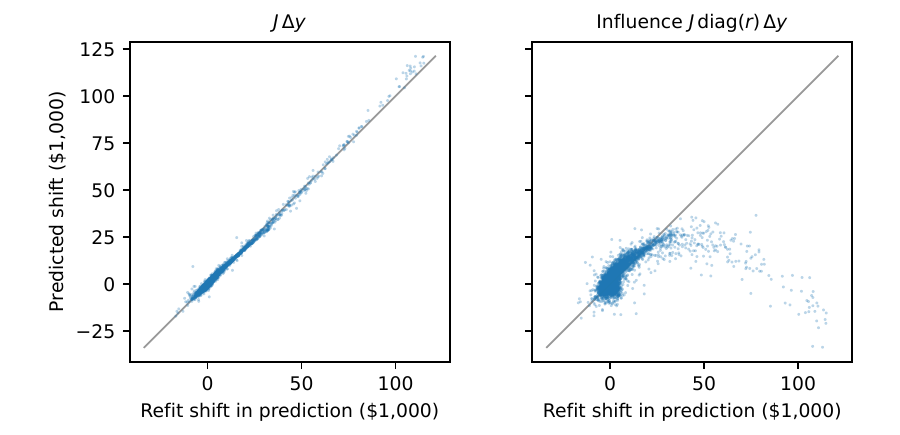}
\caption{Increasing the top-coded house values by 20\% (California
Housing, seed 0). Each point is one of 5,160 held-out block groups. The
horizontal axis is the shift in its prediction after refitting, and
the vertical axis is the shift predicted by $J$ (left) and by
influence (right). The grey line marks equality.}
\label{fig:california}
\end{figure}

In Figure~\ref{fig:california}, $J$ follows the refit over the whole
range of shifts, whereas influence misses most on the block groups
whose predictions move the most. $J$'s $R^2$ lies between 0.98 and
1.00 and influence's between 0.21 and 0.27
(Table~\ref{tab:california}). Influence's direction is only
partly right, with $\rho$ between 0.47 and 0.53
(its $\alpha$ depends on the unit of the
target, as on Friedman).

\section{Limitations}
\label{sec:discussion}

$J$ is a first-order sensitivity at a stationary point. How well it
predicts the effect of a finite target change depends on how nonlinear
the refit is in that change, which varies with the learner and the size
of the change. This limit, like the need for white-box
access and one Hessian solve per row or column, applies equally to
influence. Any difficulty in computing the solve accurately also
affects both operators alike, because the modulation identity holds
for any shared Hessian or surrogate. On MNIST, replacing the exact
Hessian by the Gauss--Newton surrogate lowers $J$'s $R^2$ from 0.92 to
0.74 (Table~\ref{tab:mnist_diag}). The modulation
identity also needs a pointwise loss with $\ell_{fy}\neq 0$, which
every canonical-link loss satisfies. Where $\ell_{fy}=0$, as in the
linear region of the Huber loss or almost everywhere for the $L_1$
loss, the modulation factor is undefined. The fit does not respond to
that target to first order, so $J_{:,j}=0$. Influence's column
$\mathcal{I}_{:,j}$, however, is in general nonzero, and no factor
$m_j$ can scale a zero column into a nonzero one.

\bibliographystyle{tmlr}
\bibliography{refs}

\newpage
\appendix

\section{Notation}
\label{app:notation}

Table~\ref{tab:notation} lists the symbols used across sections,
including those defined inside a proof. Standard notation, and symbols
used only within a single passage, proof or appendix, are defined where
they are used.

Influence $\mathcal{I}_{ij}$ is the derivative with respect to
upweighting training point $j$ in the $(1/N)$-normalised loss.
Leave-one-out work uses the opposite sign, and the additive-$\epsilon$
form of \citet{koh2017understanding} drops the $1/N$. The modulation
identity holds in those conventions up to the corresponding global sign
or scale.

\begin{table}[htbp]
\centering
\caption{Global notation, grouped by topic, with the place where each
symbol is defined or first used.}
\label{tab:notation}
\small
\renewcommand{\arraystretch}{1.1}
\begin{tabular}{@{}lp{8.6cm}l@{}}
\toprule
Symbol & Meaning & Where \\
\midrule
\multicolumn{3}{@{}l}{\emph{Data and indices}} \\
$N$, $S$, $P$ & Numbers of training points, evaluation predictions and
  parameters & Sec.~\ref{sec:identity} \\
$i\in[S]$, $j\in[N]$ & Index of an evaluation prediction and of a
  training point; $[N]:=\{1,\dots,N\}$ & Sec.~\ref{sec:identity} \\
$\by$, $\by^\star$ & Training targets $(y_1,\dots,y_N)$; the observed
  targets & Sec.~\ref{sec:identity} \\
$K_c$, $y_{jk}$ & Number of classes; class-$k$ entry of a vector target
  $y_j\in\R^{K_c}$ & Sec.~\ref{sec:identity} \\
\addlinespace
\multicolumn{3}{@{}l}{\emph{Model, loss and training}} \\
$\btheta\in\R^P$, $\btheta^\star$ & Parameters; the fitted parameters,
  a stationary point & Sec.~\ref{sec:identity} \\
$\btheta^\star(\by)$ & Locally defined solution for targets $\by$, with
  $\btheta^\star(\by^\star)=\btheta^\star$ & Sec.~\ref{sec:identity} \\
$\mathcal{L}(\btheta;\by)$, $\mathcal{L}(\btheta;\by,\bw)$ & Training
  objective; the same with sample weights $\bw$ & Sec.~\ref{sec:identity} \\
$\ell_j$, $R$ & Per-example loss; regulariser & Sec.~\ref{sec:identity} \\
$\bw$ & Sample weights, $\bw=\mathbf{1}$ at the trained model &
  Sec.~\ref{sec:identity} \\
$H$ & Hessian of $\mathcal{L}$ at $(\btheta^\star,\by^\star)$ &
  Sec.~\ref{sec:identity} \\
$f_{\btheta}(x)$, $f_j$ & Model output; $f_j:=f_{\btheta^\star}(x_j)$ &
  Sec.~\ref{sec:identity} \\
$q_i(\btheta)$ & Scalar prediction at $z_i$, such as a logit &
  Def.~\ref{def:target_response} \\
$\ell(f,y)$, $\ell_f$, $\ell_{fy}$ & Pointwise loss; its partial
  derivative in $f$ and mixed partial in $f$ and $y$ &
  Sec.~\ref{sec:identity} \\
$h$, $g$ & The functions in the canonical-link form
  $\ell(f,y)=h(f)-yf+g(y)$ & Cor.~\ref{cor:canonical_link} \\
$p_k(\btheta,x)$, $p_{j,k}$ & Softmax probability of class $k$; at
  $x_j$ and $\btheta^\star$ & Sec.~\ref{sec:identity} \\
\addlinespace
\multicolumn{3}{@{}l}{\emph{Operators}} \\
$F(\by)$, $F(\by,\bw)$ & Target-response map, the predictions as a
  function of the targets; the same with sample weights &
  Def.~\ref{def:target_response} \\
$\hat\by$ & Predictions, $\hat\by:=F(\by)$ & Eq.~\eqref{eq:local-identity} \\
$J$ & Target-response Jacobian $\partial F/\partial\by$ at $\by^\star$,
  in $\R^{S\times N}$ ($\R^{S\times N\times K_c}$ for vector targets) &
  Def.~\ref{def:target_jacobian} \\
$\mathcal{I}$ & Sample-weight influence $\partial F/\partial\bw$ at
  $(\by^\star,\mathbf{1})$, in $\R^{S\times N}$ & Sec.~\ref{sec:identity} \\
$G$, $g_i$ & $S\times P$ matrix with rows $g_i^\top$, where
  $g_i:=\nabla_{\btheta}q_i(\btheta^\star)$ & Proof of
  Thm.~\ref{thm:scalar_mod} \\
$C$, $c_j$ & $P\times N$ matrix with columns
  $c_j:=(1/N)\nabla_{\btheta}(\partial\ell_j/\partial y_j)$ & Proof of
  Thm.~\ref{thm:scalar_mod} \\
$B$, $b_j$ & $P\times N$ matrix with columns
  $b_j:=(1/N)\nabla_{\btheta}\ell_j$, the per-example loss gradients &
  Proof of Thm.~\ref{thm:scalar_mod} \\
$\psi_j$ & Output gradient $\nabla_{\btheta}f_{\btheta^\star}(x_j)$ &
  Proof of Thm.~\ref{thm:scalar_mod} \\
$m_j$, $m$ & Modulation factor $\ell_f/\ell_{fy}$ at
  $(f_j,y_j^\star)$; the vector of all $m_j$ & Eq.~\eqref{eq:mu_def} \\
$r_j$, $r$ & Training residual $y_j^\star-h'(f_j)$; the vector of all
  $r_j$ & Cor.~\ref{cor:canonical_link} \\
$r_{jk}$ & Class-$k$ residual $y_{jk}-p_k(\btheta^\star,x_j)$ &
  Prop.~\ref{prop:multiclass} \\
$\hat H$, $\hat J$, $\hat{\mathcal I}$ & Surrogate Hessian; $J$ and
  $\mathcal{I}$ computed with it & Rem.~\ref{cor:surrogate_mod} \\
$H_{\mathrm{GN}}$, $\lambda_{\mathrm{damp}}$ & Gauss--Newton Hessian;
  damping, equal to the L2 penalty & Rem.~\ref{cor:surrogate_mod} \\
$K$, $K_{\mathrm{OLS}}$, $K_{\mathrm{KRR}}$, $K_{\mathrm{AXIL}}$ &
  Smoother matrices, with $\hat\by=K\by$ exactly &
  Table~\ref{tab:lineage} \\
\addlinespace
\multicolumn{3}{@{}l}{\emph{Perturbations and evaluation}} \\
$\Delta\by$, $\bdelta$ & Target change (perturbation, if random) &
  Eq.~\eqref{eq:taylor_F} \\
$t_j$, $a$ & True (recorded) class of training point $j$; destination class of a
  label change $\delta(e_a-e_{t_j})$ &
  Cor.~\ref{cor:true_class_collapse} \\
$K$ & Number of perturbations & Sec.~\ref{sec:r2} \\
$\Delta F_i$ & Refit shift of prediction $i$ & Def.~\ref{def:r2} \\
$R^2_i$ & Per-prediction coefficient of determination &
  Def.~\ref{def:r2} \\
$\rho_i$, $\alpha_i$, $\beta_i$ & Direction, scale and bias components
  of $R^2_i$ & Sec.~\ref{sec:r2} \\
\bottomrule
\end{tabular}
\end{table}

\section{Direction, scale and bias}
\label{app:proofs}

\begin{lemma}[Mincer--Zarnowitz decomposition]
\label{lem:mincer_zarnowitz}
Let $\rho_i$, $\alpha_i$ and $\beta_i$ be as in Section~\ref{sec:r2},
with both standard deviations positive. Then
$R^2_{i} = 2\rho_i\alpha_i - \alpha_i^2 - \beta_i^2
= \rho_i^2 - (\alpha_i-\rho_i)^2 - \beta_i^2$.
\end{lemma}

\begin{proof}
Let $s:=\Delta F_i$ and $\hat s:=(M\Delta\by)_i$ under
$\Delta\by\sim\mathcal{D}$, with means $\mu_s,\mu_{\hat s}$ and
standard deviations $\sigma_s,\sigma_{\hat s}$, so that
$\alpha_i=\sigma_{\hat s}/\sigma_s$ and
$\beta_i=(\mu_{\hat s}-\mu_s)/\sigma_s$. Then
$\mathbb{E}[(s-\hat s)^2]
 = (\mu_s-\mu_{\hat s})^2 + \sigma_s^2 + \sigma_{\hat s}^2
   - 2\rho_i\sigma_s\sigma_{\hat s}
 = \sigma_s^2\bigl(\beta_i^2 + 1 + \alpha_i^2 - 2\rho_i\alpha_i\bigr)$.
Dividing by $\operatorname{Var}[s]=\sigma_s^2$ and subtracting from one
gives the result.
\end{proof}


\section{Kernel ridge regression}
\label{app:krr}

For kernel ridge regression with training Gram matrix
$K_{XX}\in\R^{N\times N}$, evaluation kernel $K_{QX}\in\R^{S\times N}$
and ridge parameter $\lambda>0$, the fit
$\hat\by = K_{QX}(K_{XX}+\lambda I)^{-1}\by$ is linear in $\by$
\citep{saunders1998ridge}, and $J = K_{QX}(K_{XX}+\lambda I)^{-1}$ is the
smoother matrix \citep[eq.~5.57]{hastie2009elements} at every $\by$.
Corollary~\ref{cor:canonical_link} gives
$\mathcal{I} = J\,\diag(r)$ with
$r_j = y_j - [K_{XX}(K_{XX}+\lambda I)^{-1}\by]_j$. As $\lambda\to 0$ with
$K_{XX}$ invertible, $r\to 0$ and $\mathcal{I}\to 0$, while
$J\to K_{QX}K_{XX}^{-1}$ stays finite. This is interpolation collapse in
closed form.

Table~\ref{tab:r2_krr} uses Friedman\,\#1 ($N=2{,}000$, $S=2{,}000$)
with an RBF kernel of bandwidth 1 on the raw features, and $K=100$
perturbations drawn as on Friedman in Appendix~\ref{app:protocol}. The fit is linear in the targets, so $J\Delta\by$
equals the refit shift and $J$'s $R^2$ is $1.000$ at every $\lambda$. Influence's $R^2$
approaches zero, the score of predicting no change, as $\lambda\to 0$.

\begin{table}[htbp]
\centering
\caption{Kernel ridge regression on Friedman\,\#1. Median per-prediction
$R^2$ over $2{,}000$ predictions.}
\label{tab:r2_krr}
\small
\input{table_r2_krr}
\end{table}


\section{Softmax cross-entropy}
\label{app:multiclass}

\begin{proof}[Proof of Proposition~\ref{prop:multiclass}]
Let $\psi_{jk} := \nabla_{\btheta} f_k(\btheta^\star,x_j)$. For
$\ell_j = h(f(\btheta,x_j)) - y_j^\top f(\btheta,x_j)$ with
$h(f)=\log\sum_{k'}e^{f_{k'}}$, the target derivative in class $k$ is
$\partial\ell_j/\partial y_{jk} = -f_k$, so
$c_{jk} = (1/N)\nabla_{\btheta}(\partial\ell_j/\partial y_{jk}) = -(1/N)\psi_{jk}$.
The loss gradient is
$\nabla_{\btheta}\ell_j = \sum_k(p_k - y_{jk})\psi_{jk} = -\sum_k r_{jk}\psi_{jk}$,
so $b_j = \sum_k r_{jk}c_{jk}$. Substituting into
$\mathcal{I}_{:,j} = -GH^{-1}b_j$ and using $J_{:,j,k} = -GH^{-1}c_{jk}$
gives $\mathcal{I}_{:,j} = \sum_k r_{jk}J_{:,j,k}$. The same steps with
$\hat H$ in place of $H$ give the surrogate version.
\end{proof}

With $K_c=2$ and the single logit $f:=f_1-f_2$ this reduces to the
logistic case of Corollary~\ref{cor:canonical_link}.

Class probabilities sum to one, so only perturbations with
$\sum_k \Delta y_{jk} = 0$ give valid targets. The partials $J_{:,j,k}$
ignore this constraint, but \eqref{eq:multiclass} is unaffected:
because the residuals also satisfy $\sum_k r_{jk}=0$, any part of
$J_{:,j,k}$ that is common to all classes $k$ drops out of the sum.

\paragraph{Scoring influence on label changes.}
$\mathcal{I}_{:,j}$ is one vector per training point, while a label
change $\Delta\by_j=\delta(e_a-e_t)$ is a vector in $\R^{K_c}$. Scoring
influence therefore needs a convention for sign and magnitude. We use
$\delta\,\mathcal{I}_{:,j}$, which treats the label change as
upweighting training point $j$ by $\delta$, the same scalar mass that
enters $J$'s prediction $\sum_k J_{:,j,k}\Delta y_{jk}$, with no
rescaling.
Scaling by the norm $\sqrt2\,|\delta|$ of the label change instead
would multiply influence's $\alpha_i$ by $\sqrt2$ and leave $\rho_i$
unchanged. Its $R^2$ would stay negative wherever $\rho_i<0$, since
$2\rho_i\alpha_i-\alpha_i^2-\beta_i^2<0$ for any $\alpha_i>0$.

\paragraph{Why influence's direction is weak.}
Factoring $(1-p_{j,t_j})$ out of \eqref{eq:true_class_decomp},
\[
  \mathcal{I}_{:,j}
  = (1-p_{j,t_j})\bigl[J_{:,j,t_j} - \bar J_{:,j}\bigr],
  \qquad
  \bar J_{:,j} := \sum_{k\neq t_j} \tfrac{p_{j,k}}{1-p_{j,t_j}}\,J_{:,j,k},
\]
with $\bar J_{:,j}$ a probability-weighted average of the other-class
columns. Write $u_i:=J_{i,j,t_j}-\bar J_{i,j}$ for the true-class term
and $v_i:=J_{i,j,a}-\bar J_{i,j}$ for the destination term. Then $J$'s
shift is $\delta(v_i-u_i)$ and influence's is
$\delta(1-p_{j,t_j})u_i$, up to the sign convention above. Influence
carries the true-class term $u_i$ but not the destination term $v_i$.
Treating $1-p_{j,t_j}$ as roughly constant across perturbations, and
with $v_i$ uncorrelated with $u_i$, as a random destination class makes
plausible, the correlation between influence and $J$ has magnitude
$\sigma_u/\sqrt{\sigma_u^2+\sigma_v^2}$ under either sign, where
$\sigma_u$ and $\sigma_v$ are the standard deviations of $u_i$ and
$v_i$ over the perturbations. It is well below one whenever the
destination term varies as much as the true-class term. On MNIST,
influence's $|\rho_i|$ against the refit, which $J$ tracks, is 0.37
(Table~\ref{tab:r2_mlp}).


\section{Protocol}
\label{app:protocol}

\paragraph{Tabular networks.}
Friedman\,\#1, MAGIC Gamma and HIGGS use a one-hidden-layer tanh
network of width 64 (Friedman, MAGIC) or 128 (HIGGS). The smooth
activation makes the loss twice continuously differentiable in
$\btheta$, as the theory requires.
Training is full-batch Adam ($\mathrm{lr}=5\cdot 10^{-3}$, up to 4,000
steps on Friedman and MAGIC and 2,000 on HIGGS, with early exit once the
training loss falls below $10^{-4}$) followed by L-BFGS with
strong-Wolfe line search until the stopping rule below is met. The objective
includes the penalty $\tfrac12\lambda_{\mathrm{damp}}\|\btheta\|^2$
with $\lambda_{\mathrm{damp}}=10^{-2}$. On Friedman the per-example loss
is $(f-y)^2$, twice the canonical form $\tfrac12(f-y)^2$, so there
$h''=2$ and $c_j=-(2/N)\psi_j$. This leaves $m_j=r_j$ unchanged.

Friedman\,\#1 has 10 features drawn uniformly on $[0,1]$ and Gaussian
noise with standard deviation 0.5, and each seed draws a new sample.
MAGIC uses the 19,020 rows of the OpenML copy, and each seed draws a new
stratified split with a quarter of the rows held out for evaluation.
The HIGGS split is deterministic in file order. The first $10^6$ rows of
the UCI file form the training set and the next $10^4$ the evaluation
set. MAGIC and HIGGS features are standardised with training-set
statistics. The seed also sets the initialisation and the
perturbations. On MAGIC and HIGGS the prediction $q_i$ is the logit of
evaluation point $i$, so all shifts are in logits.

Friedman and MAGIC run in float64 on the CPU, HIGGS in float64 on the
GPU. Both operators use the exact Hessian $H$ of the regularised
training objective \eqref{eq:exact_hessian}. With $P=769$ (Friedman,
MAGIC) or $3{,}841$ (HIGGS) parameters, $H$ is formed as a dense
$P\times P$ matrix from Hessian-vector products, checked to be
positive definite, and factorised by Cholesky, so every solve is
direct. Influence is computed from its
definition, with the loss gradients $B$ and its own solves, so it agrees
with $J\diag(r)$ up to solver error. For comparison, both operators
are also computed with the damped Gauss--Newton Hessian, in float32 by
conjugate gradients to a relative solve error of $10^{-4}$. Table~\ref{tab:tabular_diag} gives
the size of each run and its fit and solver diagnostics.

\begin{table}[htbp]
\centering
\caption{Tabular runs: sizes, stopping rules and solver checks, pooled
over seeds. $\|C\Delta\by\|$ is the change in the training-loss
gradient caused by one perturbation. The identity row is the relative
error of the canonical-link identity $B=C\operatorname{diag}(r)$ on
the computed loss gradients. The last three rows use the damped
Gauss--Newton Hessian in place of $H$.}
\label{tab:tabular_diag}
\small
\input{table_tabular_diag}
\end{table}

\paragraph{Gauss--Newton Hessian.}
$H$ can be indefinite for nonlinear models, and a common surrogate
replaces the data term by the Gauss--Newton Hessian
\begin{equation}
  H_{\mathrm{GN}}
  := \tfrac{1}{N}\sum_{j=1}^{N}
    h''\bigl(f_{\btheta^\star}(x_j)\bigr)\psi_j\psi_j^\top,
\end{equation}
with $h''=1$ for squared error in the canonical form ($h''=2$ on
Friedman) and $h''=\sigma(f)(1-\sigma(f))$ for logistic loss. For softmax cross-entropy the scalar $h''$ becomes the
softmax Jacobian $\partial p_j/\partial f_j=\diag(p_j)-p_jp_j^\top$ and $\psi_j$ becomes the
$P\times K_c$ Jacobian of the logits. Adding the penalty Hessian
$\lambda_{\mathrm{damp}}I$ gives
$\hat J := -G(H_{\mathrm{GN}}+\lambda_{\mathrm{damp}}I)^{-1}C$. For
linear-in-parameters squared error this is the exact $J$. In general,
for a canonical-link loss with the L2 penalty, the exact Hessian is
\begin{equation}
  H = H_{\mathrm{GN}} + \lambda_{\mathrm{damp}}I
  - \frac1N\sum_{j=1}^N r_j\,\nabla_{\btheta}^2 f_{\btheta^\star}(x_j),
  \label{eq:exact_hessian}
\end{equation}
so Gauss--Newton drops the residual-weighted curvature of the model
output. For softmax cross-entropy the last term becomes
$-(1/N)\sum_{j,k} r_{jk}\nabla_{\btheta}^2 f_k(\btheta^\star,x_j)$. The
dropped term is small when the residuals or the output curvature are
small, but stationarity alone does not control it.

\paragraph{Tabular perturbations.}
Each perturbation moves one training target, $\Delta\by^{(k)}=\delta_k
e_{j_k}$, with $j_k$ uniform on $[N]$. On Friedman
$\delta=s\,u\,\sigma_y$, with sign $s$ uniform on $\{-1,+1\}$,
$u\sim\mathrm{Unif}[0.1,1.0]$ and $\sigma_y$ the standard deviation of
the training targets. On MAGIC and HIGGS
the soft-target perturbation is $\delta = (1-2y_j)\,\nu$ with
$\nu\sim\mathrm{Unif}[0.01,0.30]$, which keeps $y_j+\delta$ in
$[0.01,0.99]$. This ties the sign of $\delta$ to the training point's
class, so it is not
the symmetric distribution of Corollary~\ref{cor:sign_washout}. Washout
still applies because logistic residuals have class-dependent signs
($r_j>0$ for positives, $r_j<0$ for negatives). Influence's $\rho$
near zero on both datasets (Table~\ref{tab:r2_mlp}) confirms this.

\paragraph{Refits.}
Each perturbed problem is refit by warm-started L-BFGS from
$\btheta^\star$ with the same L2 penalty and the same stopping rule,
with an iteration cap. All estimators are scored against the same refits.

\paragraph{Stopping rule.}
A perturbation changes the training-loss gradient at $\btheta^\star$
by $\|C\Delta\by\|$, which shrinks as $1/N$. A fit that stops with a
gradient norm of that size leaves a refit error as large as the
response being measured. Every fit, base or refit, therefore runs
until the gradient norm of the full objective is below
\begin{equation*}
  \tau = \min\bigl(10^{-6},\; 0.05\cdot\operatorname{median}_k \|C\Delta\by^{(k)}\|\bigr),
\end{equation*}
with the median over the perturbations of the seed (on California
Housing, over its four target changes). On Friedman,
MAGIC, MNIST and California Housing $\tau=10^{-6}$; on HIGGS, where
$N=10^6$, it is smaller. Tables~\ref{tab:tabular_diag} and
\ref{tab:mnist_diag} give the largest
gradient norms and the largest remaining gradient of a refit as a
fraction of its own $\|C\Delta\by\|$. Both also count the
refits that stop at the iteration cap above $\tau$.
Because $\tau$ is set from the median, a perturbation whose
$\|C\Delta\by\|$ is far below the median, from a small $|\delta|$ or a
training point with a small output gradient, can stop at a gradient of
the same order as its own $\|C\Delta\by\|$. This happens on MAGIC and
HIGGS (Table~\ref{tab:tabular_diag}). Such a refit meets the same
absolute tolerance as every other refit, and its shift is
correspondingly small, so it adds little to the variance of refit
shifts in the denominator of $R^2_i$. Refit error that is unrelated to
$J$'s own error lowers $J$'s measured $R^2$, so $J$'s $R^2$ of 0.99 and
above on these datasets also bounds how much the refit error can
matter.

\paragraph{MNIST.}
Images are scaled to $[0,1]$ and flattened, and pixels with zero
variance in the training set are dropped, which leaves $d$ inputs
(Table~\ref{tab:mnist_diag}). The network is
$\mathrm{Linear}(d,64)\to\tanh\to\mathrm{Linear}(64,10)$,
trained in float64 on the GPU by full-batch Adam (1,200 epochs,
$\mathrm{lr}=5\cdot 10^{-3}$), then L-BFGS until the stopping rule is
met, with the same L2 penalty $\lambda_{\mathrm{damp}}=10^{-2}$. Refits
are warm-started L-BFGS on all parameters to the same tolerance. A
refit takes 6 to 34 seconds, averaged over the perturbations of a seed,
so each seed uses $K=20$ perturbations.

Both operators use the exact Hessian $H$ of the regularised training
objective \eqref{eq:exact_hessian}, which differs from
$H_{\mathrm{GN}}+\lambda_{\mathrm{damp}}I$ by the residual-weighted
curvature term. Each solve uses MINRES, since $H$ need not be positive
definite for a nonconvex network, and is checked to reach a relative
solve error $\|Hu-v\|/\|v\|$ below $10^{-6}$.
Table~\ref{tab:mnist_diag} gives the largest such error and the
smallest $u^\top Hu/u^\top u$ over the solutions $u$ for $J$, which is
positive. Influence is computed from its definition, with right-hand
side $b_j$ built from the loss gradient and its own MINRES solve, and
the identity $b_j=\sum_k r_{jk}c_{jk}$ is checked on the right-hand
side. For comparison, $J$ is also computed with the Gauss--Newton
Hessian, by conjugate gradients in float64 to a relative solve error
of $10^{-8}$, on the same perturbations.

\begin{table}[htbp]
\centering
\caption{MNIST: sizes, stopping tolerances and solver checks, pooled
over the five seeds. $\|C\Delta\by\|$ is the change in the training-loss
gradient caused by the label change, and $u$ is the exact-Hessian
solution for one label change. The last row gives $J$'s $R^2$
with the Gauss--Newton Hessian.}
\label{tab:mnist_diag}
\small
\input{table_mnist_diag}
\end{table}


\section{California Housing}
\label{app:california}

We use the scikit-learn copy of California Housing
\citep{pace1997sparse}, with 20,640 rows and 8 features. The target is the median house
value in units of \$100{,}000, and 965 rows have the top-coded value
\$500{,}001. A random quarter of the
rows are held out for evaluation (seed 0). Features are standardised with
training-set statistics, and the target keeps its units. The network,
loss, penalty, training and refit rules are those of Friedman
(Appendix~\ref{app:protocol}), in float64 on the CPU.

Both operators use the exact Hessian, formed and factorised as on
Friedman ($P=641$). For each increase $\kappa$, $J\Delta\by$ is one
direct solve with right-hand side $C\Delta\by$. Influence is computed from its
definition, with right-hand side
$B\Delta\by=\tfrac1N\sum_j\Delta y_j\nabla_{\btheta}\ell_j$ built from
the loss gradients and its own solve. Each target change is refit from
$\btheta^\star$. $R^2$, $\rho$, $\alpha$ and $\beta$ are those of
\eqref{eq:r2} and \eqref{eq:r2_decomp}, with the averages taken over the
evaluation predictions of one target change instead of over
perturbations at one prediction.


\section{Compute}
\label{app:timing}

$J$ is never formed. A column $J_{:,j}$ costs one Hessian solve, and a
row $J_{i,:}$ costs one solve plus one Jacobian-vector product over the
training set. On the tabular data and California Housing, the solve
uses the Cholesky factor of the dense exact Hessian, which is formed
once per fitted model from Hessian-vector products over the training
set (Appendix~\ref{app:protocol}). On MNIST, where $P$ is about
41,000, the solve uses MINRES against Hessian-vector products, and $H$
is never formed. On HIGGS ($N=10^6$, $S=10^4$), one row
$J_{i,:}\in\R^{10^6}$, which is one prediction's sensitivity to every
training target, takes between 0.07 and 0.14 seconds (median 0.07) on
one RTX~3090 across 30 rows, after a one-time setup of under five
minutes that forms and factorises $H$ (last lines of
Table~\ref{tab:timing}). Computing the same row
by refitting would take $10^6$ refits, one per training target.

All experiments run on one workstation with an Intel Core i9-14900K
(PyTorch capped at four threads), 192\,GB of RAM and an NVIDIA RTX~3090
(24\,GB, CUDA 12.6). Friedman and MAGIC run on the CPU, and HIGGS and
MNIST on the GPU. Base fits and refits run in float64, which the refits
need to resolve small target changes. $J$ and influence with the exact
Hessian also run in float64, and the Gauss--Newton comparison runs in
float32 on the tabular data and in float64 on MNIST. The refits
dominate the wall-clock in Table~\ref{tab:timing}. The HIGGS row
timing above comes from a separate run that times 10 rows on each of
three trained HIGGS models.

\begin{table}[htbp]
\centering
\small
\caption{Wall-clock per seed in seconds, mean $\pm$ SD over five
seeds. The last two lines time single rows of $J$ on HIGGS and the one-time
setup before them. Base fit is one training run. $J$ columns covers
forming and factorising the exact Hessian and the solves for the
sampled columns. On MNIST it covers the MINRES solves for $J$. $K$ refits covers the refits used as ground truth.}
\label{tab:timing}
\input{table_timing}
\end{table}

\end{document}

%% file: table_r2_mlp.tex
\begin{tabular}{@{}llcccc@{}}
\toprule
Dataset & Method & $R^2$ & $\rho$ & $\alpha$ & $\beta$ \\
\midrule
Friedman\,\#1 & $J$ & $\mathbf{0.998 {\scriptstyle \pm 0.001}}$ & $0.999 {\scriptstyle \pm 0.001}$ & $1.000 {\scriptstyle \pm 0.001}$ & $-0.001 {\scriptstyle \pm 0.001}$ \\
 & Influence & $-0.416 {\scriptstyle \pm 0.162}$ & $0.023 {\scriptstyle \pm 0.098}$ & $0.675 {\scriptstyle \pm 0.047}$ & $-0.007 {\scriptstyle \pm 0.011}$ \\
 & Random & $-0.996 {\scriptstyle \pm 0.023}$ & $0.002 {\scriptstyle \pm 0.011}$ & $1.000 {\scriptstyle \pm 0.000}$ & $0.000 {\scriptstyle \pm 0.000}$ \\
\midrule
MAGIC Gamma & $J$ & $\mathbf{1.000 {\scriptstyle \pm 0.000}}$ & $1.000 {\scriptstyle \pm 0.000}$ & $1.000 {\scriptstyle \pm 0.001}$ & $-0.001 {\scriptstyle \pm 0.001}$ \\
 & Influence & $-0.203 {\scriptstyle \pm 0.061}$ & $-0.082 {\scriptstyle \pm 0.118}$ & $0.301 {\scriptstyle \pm 0.025}$ & $0.099 {\scriptstyle \pm 0.045}$ \\
 & Random & $-0.984 {\scriptstyle \pm 0.065}$ & $0.008 {\scriptstyle \pm 0.032}$ & $1.000 {\scriptstyle \pm 0.000}$ & $0.000 {\scriptstyle \pm 0.000}$ \\
\midrule
HIGGS & $J$ & $\mathbf{0.991 {\scriptstyle \pm 0.004}}$ & $0.996 {\scriptstyle \pm 0.002}$ & $1.017 {\scriptstyle \pm 0.008}$ & $-0.005 {\scriptstyle \pm 0.003}$ \\
 & Influence & $-0.208 {\scriptstyle \pm 0.095}$ & $0.026 {\scriptstyle \pm 0.099}$ & $0.467 {\scriptstyle \pm 0.021}$ & $-0.053 {\scriptstyle \pm 0.011}$ \\
 & Random & $-1.017 {\scriptstyle \pm 0.050}$ & $-0.008 {\scriptstyle \pm 0.025}$ & $1.000 {\scriptstyle \pm 0.000}$ & $0.000 {\scriptstyle \pm 0.000}$ \\
\midrule
MNIST & $J$ & $\mathbf{0.924 {\scriptstyle \pm 0.051}}$ & $0.972 {\scriptstyle \pm 0.021}$ & $0.994 {\scriptstyle \pm 0.019}$ & $-0.018 {\scriptstyle \pm 0.006}$ \\
 & Influence & $-0.203 {\scriptstyle \pm 0.047}$ & $-0.373 {\scriptstyle \pm 0.088}$ & $0.101 {\scriptstyle \pm 0.040}$ & $0.303 {\scriptstyle \pm 0.042}$ \\
 & Random & $-0.930 {\scriptstyle \pm 0.287}$ & $0.035 {\scriptstyle \pm 0.143}$ & $1.000 {\scriptstyle \pm 0.000}$ & $0.000 {\scriptstyle \pm 0.000}$ \\
\bottomrule
\end{tabular}

%% file: table_california.tex
\begin{tabular}{@{}llcccc@{}}
\toprule
Increase & Method & $R^2$ & $\rho$ & $\alpha$ & $\beta$ \\
\midrule
5\% & $J$ & $1.000$ & $1.000$ & $1.002$ & $0.000$ \\
 & Influence & $0.213$ & $0.470$ & $0.549$ & $-0.040$ \\
\midrule
10\% & $J$ & $0.996$ & $0.998$ & $1.002$ & $0.001$ \\
 & Influence & $0.215$ & $0.471$ & $0.549$ & $-0.040$ \\
\midrule
20\% & $J$ & $0.994$ & $0.997$ & $1.009$ & $0.000$ \\
 & Influence & $0.237$ & $0.492$ & $0.553$ & $-0.041$ \\
\midrule
50\% & $J$ & $0.977$ & $0.989$ & $1.011$ & $0.000$ \\
 & Influence & $0.274$ & $0.525$ & $0.554$ & $-0.041$ \\
\bottomrule
\end{tabular}

%% file: table_r2_krr.tex
\begin{tabular}{lcc}
\toprule
$\lambda$ & $R^2$ ($J$) & $R^2$ (Influence) \\
\midrule
$10^{-1}$ & $+1.000$ & $+0.131$ \\
$10^{-3}$ & $+1.000$ & $-0.062$ \\
$10^{-6}$ & $+1.000$ & $-0.002$ \\
\bottomrule
\end{tabular}

%% file: table_tabular_diag.tex
\begin{tabular}{@{}lrrr@{}}
\toprule
 & Friedman\,\#1 & MAGIC Gamma & HIGGS \\
\midrule
Training points $N$ & $2{,}000$ & $14{,}265$ & $1.0{\cdot}10^6$ \\
Evaluation predictions $S$ & $2{,}000$ & $4{,}755$ & $10{,}000$ \\
Parameters $P$ & $769$ & $769$ & $3{,}841$ \\
Perturbations $K$ per seed & $100$ & $100$ & $100$ \\
Seeds & 5 & 5 & 5 \\
Stopping tolerance on the gradient norm, smallest & $1.0{\cdot}10^{-6}$ & $1.0{\cdot}10^{-6}$ & $3.8{\cdot}10^{-8}$ \\
Base fits: largest gradient norm & $1.0{\cdot}10^{-6}$ & $8.1{\cdot}10^{-7}$ & $3.5{\cdot}10^{-8}$ \\
Refits at or above the tolerance & 8 of 500 & 0 of 500 & 0 of 500 \\
Refits: largest gradient norm / $\|C\Delta\mathbf{y}\|$ & $2.5{\cdot}10^{-4}$ & $2.8{\cdot}10^{-1}$ & $7.3{\cdot}10^{-1}$ \\
Exact Hessian: smallest eigenvalue & $8.7{\cdot}10^{-4}$ & $1.1{\cdot}10^{-4}$ & $9.6{\cdot}10^{-5}$ \\
Direct solves: largest relative error & $7.3{\cdot}10^{-14}$ & $3.5{\cdot}10^{-15}$ & $1.3{\cdot}10^{-14}$ \\
$\|B - C\operatorname{diag}(r)\| / \|B\|$, largest & $2.9{\cdot}10^{-14}$ & $6.2{\cdot}10^{-14}$ & $9.3{\cdot}10^{-16}$ \\
Gauss--Newton CG iterations, largest & $128$ & $26$ & $31$ \\
$J$, Gauss--Newton Hessian: $R^2$ & $0.928 {\scriptstyle \pm 0.014}$ & $0.949 {\scriptstyle \pm 0.015}$ & $0.956 {\scriptstyle \pm 0.015}$ \\
Influence, Gauss--Newton Hessian: $R^2$ & $-0.489 {\scriptstyle \pm 0.190}$ & $-0.213 {\scriptstyle \pm 0.060}$ & $-0.190 {\scriptstyle \pm 0.085}$ \\
\bottomrule
\end{tabular}

%% file: table_mnist_diag.tex
\begin{tabular}{@{}lr@{}}
\toprule
Training points $N$ & $2{,}000$ \\
Evaluation predictions $S$ & $2{,}000$ \\
Retained pixels $d$ & $629$--$640$ \\
Parameters $P$ & $40{,}970$--$41{,}674$ \\
Perturbations $K$ per seed & $20$ \\
Seeds & $5$ \\
Test accuracy & $0.906$--$0.919$ \\
Training points: median $p_{j,t_j}$ & $0.920$--$0.931$ \\
Base fits: largest gradient norm & $9.9{\cdot}10^{-7}$ \\
Refits at or above $10^{-6}$ & 1 of 100 \\
Refits: largest gradient norm / $\|C\Delta\mathbf{y}\|$ & $8.3{\cdot}10^{-3}$ \\
Gauss--Newton CG iterations, largest & $86$ \\
MINRES iterations, median per seed & $626$--$875$ \\
MINRES relative solve error, largest & $2.1{\cdot}10^{-7}$ \\
$\min u^\top H u / u^\top u$ & $7.1{\cdot}10^{-5}$ \\
$\|B - C\operatorname{diag}(r)\| / \|B\|$, largest & $9.2{\cdot}10^{-14}$ \\
$J$, Gauss--Newton Hessian: $R^2$ & $0.739 {\scriptstyle \pm 0.065}$ \\
\bottomrule
\end{tabular}

%% file: table_timing.tex
\begin{tabular}{lrcccl}
\toprule
Dataset & $N_{\mathrm{train}}$ & Base fit (s) & $J$ columns (s) & $K$ refits (s) \\
\midrule
Friedman\,\#1 (CPU) & $2{,}000$ & $7.2{\scriptstyle \pm 0.7}$ & $0.35{\scriptstyle \pm 0.00}$ & $73{\scriptstyle \pm 91}$\,\scriptsize($K{=}100$) \\
MAGIC Gamma (CPU) & $14{,}265$ & $7.7{\scriptstyle \pm 0.4}$ & $2.1{\scriptstyle \pm 0.0}$ & $14{\scriptstyle \pm 2}$\,\scriptsize($K{=}100$) \\
HIGGS (GPU) & $1.0{\cdot}10^6$ & $337{\scriptstyle \pm 141}$ & $263{\scriptstyle \pm 0}$ & $444{\scriptstyle \pm 94}$\,\scriptsize($K{=}100$) \\
MNIST (GPU) & $2{,}000$ & $21{\scriptstyle \pm 7}$ & $138{\scriptstyle \pm 22}$ & $379{\scriptstyle \pm 218}$\,\scriptsize($K{=}20$) \\
\midrule
\multicolumn{5}{@{}l}{HIGGS (GPU), one row of $J$: median 0.07\,s, range 0.07--0.14\,s, 30 rows} \\
\multicolumn{5}{@{}l}{HIGGS (GPU), one-time setup per model: median 271.11\,s, range 263.72--271.34\,s} \\
\bottomrule
\end{tabular}